\documentclass{article}
\usepackage[utf8]{inputenc}
\usepackage{enumitem, algorithm, algorithmic, fullpage, natbib, cite}
\usepackage{amsmath,amsfonts,amsthm,  mathrsfs,mathtools}
\usepackage{xcolor}
\usepackage{float}
\usepackage{hyperref}
\usepackage{booktabs}
\usepackage{subfig}

\newcommand\blfootnote[1]{%
  \begingroup
  \renewcommand\thefootnote{}\footnote{#1}%
  \addtocounter{footnote}{-1}%
  \endgroup
}

\theoremstyle{plain}
\newtheorem{theorem}{Theorem}[section]
\newtheorem{lemma}[theorem]{Lemma}
\newtheorem{corollary}[theorem]{Corollary}

\newtheorem{remark}[theorem]{Remark}

\theoremstyle{definition}   
\newtheorem{definition}[theorem]{Definition}

\author{
 George Z. Li$^{2,*}$\quad
 Ann Li$^{1}$\quad 
 Madhav Marathe$^{1,3}$ \and
 Aravind Srinivasan$^{2}$ \quad
 Leonidas Tsepenekas$^{2}$ \quad
 Anil Vullikanti$^{1,3}$
}

\date{}

\newcommand{\prob}{\textsc{MobileVaccClinic}}

\newcommand{\cover}{\textsc{ClientCover}}

\newcommand{\fpt}{\textsc{FPT}}
\newcommand{\homecenter}{\textsc{HomeCenters}}
\newcommand{\pop}{\textsc{MostActive}}

\DeclareMathOperator*{\poly}{poly}

\DeclarePairedDelimiter\floor{\lfloor}{\rfloor}

\title{Deploying Vaccine Distribution Sites for Improved Accessibility and Equity to Support Pandemic Response}

\begin{document}

\maketitle

\blfootnote{$^1$Biocomplexity Institute and Initiative, University of Virginia}
\blfootnote{$^2$Department of Computer Science, University of Maryland}
\blfootnote{$^3$Department of Computer Science, University of Virginia}
\blfootnote{$^*$Correspondence to George Li at gzli929@gmail.com}

\begin{abstract}
In response to COVID-19, many countries have mandated social distancing and banned large group gatherings in order to slow down the spread of SARS-CoV-2. These social interventions along with vaccines remain the best way forward to reduce the spread of SARS CoV-2. In order to increase vaccine accessibility, states such as Virginia have deployed mobile vaccination centers to distribute vaccines across the state. When choosing where to place these sites, there are two important factors to take into account: accessibility and equity. We formulate a combinatorial problem that captures these factors and then develop efficient algorithms with theoretical guarantees on both of these aspects. Furthermore, we study the inherent hardness of the problem, and demonstrate strong impossibility results. Finally, we run computational experiments on real-world data to show the efficacy of our methods. 
\end{abstract}

\section{Introduction}

The COVID-19 pandemic continues to cause immense social, health, and economic impact globally. As of writing this paper, the U.S. alone has seen over 850,000 deaths and over 65 million confirmed cases; see \citep{vdhDashboard} for the latest numbers. Vaccines have proven to be very effective in reducing the health burden of the pandemic and continue to be the best strategy to control disease spread and potentially end the pandemic in its current form. Despite the effectiveness, administering COVID-19 vaccines to all eligible individuals in the population continues to be a challenge.
As of February 2022, only 64\% of the eligible population is fully vaccinated in the United States~\citep{nyt}. Furthermore, there is a significant disparity in vaccination rates between demographics---the rate among Whites was 1.2 times that of African Americans and 1.1 times that of Hispanic people. The reasons why some people have not been vaccinated include distrust and skepticism regarding COVID-19, accessibility issues, and concerns about the cost~\citep{nytimes-may2021}. Lottery schemes, mandates, vaccine clinics, and other strategies have been implemented to increase the vaccination rate with varying levels of success. Since cost and accessibility remain a challenge for a fraction of the population, especially minorities and people in poorer neighborhoods, mobile vaccine clinics have been an important part of the public health response strategy of government agencies. In this paper, we study the problem of deploying mobile vaccine administration sites with the goal of improving the accessibility of vaccines to individuals.

Deploying vaccination clinics is a form of a facility location problem~\citep{turkoglu:faclocsurvey, drezner-book}, referred to as the $k$-supplier problem, in which a limited set of $k$ facilities needs to be placed so that every person (i.e., a client) is ``close'' to a facility; a common metric to measure closeness is the maximum distance between a client and their closest facility, though many other notions have been studied. Facility location problems are well understood, and efficient approximation algorithms and practical heuristics exist.
However, deploying vaccine clinics leads to a novel facility location problem (referred to as the \prob{} problem) since people (clients) are mobile rather than stationary. Suppose each person $p$ visits a set $S_p$ of locations during the day; then it suffices to deploy a clinic close to at least one location in $S_p$. Our contributions are the following: 

\begin{itemize}
\item 
We formalize the \prob{} problem for modeling the deployment of mobile vaccine clinics in a way that takes into account human mobility patterns (by considering the distance to a facility from any of the locations visited by a person), fairness (by requiring that at least a fraction of people in each demographic group have a nearby clinic), outliers (by allowing partial coverage), and capacity constraints (by restricting the number of people assigned to each clinic). We show that this problem is much harder than the standard $k$-supplier problem and getting any bounded polynomial-time approximation to the minimum distance is not possible, thus motivating bicriteria algorithms.
\item
We design two approximation algorithms. The first is a fixed-parameter tractable algorithm that gives a $3$-approximation, where the parameterization is on the number $u$ of locations where people travel. Note that even this is non-trivial, since the possible locations $\mathcal{S}$ where we can place facilities is still variable, so there are still $|\mathcal{S}| \choose k$ possible solutions. The second algorithm, based on covering problems, is a $(1,\log{n}+1)$-bicriteria one, where $n$ is the number of people. This means that if we violate the budget on the number of vaccine centers by a $\log{n}+1$ multiplicative factor, we can find a solution that is optimal. Finally, we extend our algorithms to have fairness guarantees in both the original and outliers formulation of the problem.
\item
We evaluate our algorithms for a realistic population of a county in Virginia. We find that our algorithms generally give a significant improvement over natural baselines. In particular, we see many shortcomings of only considering a client's home (rather than their entire travelling route), emphasizing the importance of our problem formulation. Additionally, our algorithms allow us to compute a tradeoff between the maximum distance to a clinic and the number of clinics; this naturally enables us to give a recommendation to the government on the most cost-effective budget policy. Finally, the solutions computed by our algorithm have a useful ``kernel'' property---as the budget is increased, the locations which were picked for a lower budget are still part of the solution. This implies that an incrementally constructed solution (which is how such facilities would be deployed in practice since the budget is not known ahead of time) will still be good.   
\end{itemize}

We remark that though our framework is motivated by the current COVID-19 pandemic, 
it can be generally applied to both epidemiological and non-epidemiological settings. Examples within healthcare include placing testing and treatment units (as deployed during the Ebola crisis) and delivering healthcare in rural settings for resource-limited countries. Beyond healthcare, the placement of mobile distribution centers arises in disaster-management settings. For instance, shelters need to be set up for individuals evacuating during a hurricane or forest fire, who might need food and other basic survival kits. During such large events, mobile sites are also used to place security posts and information kiosks.

\section{Preliminaries}
\label{sec:prelim}
Recall that we wish to place vaccination centers such that vaccines are more accessible to the population. This question is often formulated as an appropriate variant of the facility location problem, which is well-studied in the operations research literature (see Related Work). In our paper, we introduce a new variant that follows a recent line of work on integrating the mobility patterns of the population into disease models \citep{chang2021mobility, wang2020using}. As is standard, we will use the distance from a vaccination center as the metric for defining accessibility. The key change, however, is that clients will be represented by a set of locations that they visit (within a time period) instead of just one point. Though this will make the problem much harder to solve efficiently, it will more strongly correlate with the likelihood of a person going to a vaccine center.

\begin{figure}
    \centering
    \includegraphics[scale=0.23]{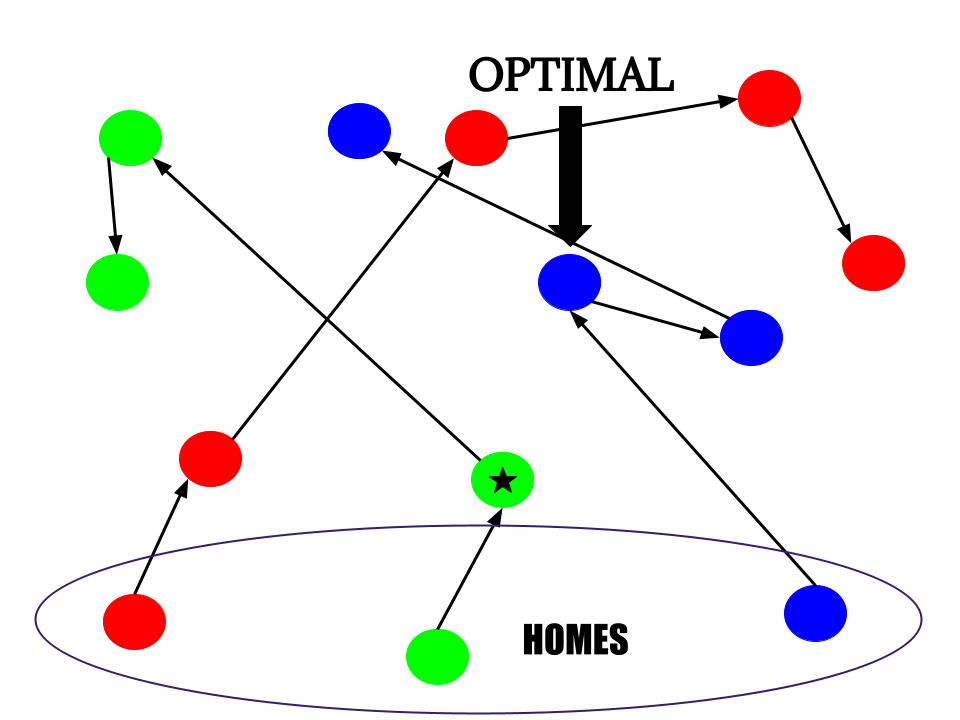}
    \caption{An example of \prob{}. The different colors represent different people and the circles represent the locations they visit (with the bottom three being their homes). In this case, the blue location in the middle is an optimal location to place a vaccination center. If we instead only considered homes in the problem formulation, we would place the vaccination center in the green circle marked with a star, which would require people to deviate from their normal travels much more when getting a vaccine.}
    \label{fig:example}
\end{figure}

\textbf{Problem Statement:} We are given a set of locations $\mathcal{C}$ in a metric space characterized by the distance function $d: \mathcal{C} \times \mathcal{C} \mapsto \mathbb{R}_{\geq 0}$. We additionally have a set of $n$ individuals/clients $\mathcal{P}$. Each individual $p \in \mathcal{P}$ is associated with a set $S_p \subseteq \mathcal{C}$, which we can interpret as the set of locations $p$ visits throughout the day. Finally, the input also includes a positive integer $k$ constraining the number of facilities we can place, and a set $\mathcal{S}\subseteq\mathcal{C}$ containing the locations where we are allowed to place facilities. The goal of \prob{} is to choose a set $F \subseteq \mathcal{S}$ with $|F|\le k$ to place facilities, such that for every $p \in \mathcal{P}$ we have $d(S_p, F) \leq R$, for the minimum $R$ possible. Here, we use the standard notation where $d(S, F) = \min_{j \in S, j' \in F}d(j,j')$. Intuitively, this objective tries to minimize the maximum distance between the set of facilities placed and the locations visited by any client. We also consider three natural extensions:
\begin{itemize}
    \item \textbf{Outliers}: in order to achieve herd immunity, we only need to vaccinate a large portion of the population (rather than every single person). In order to model this, we can take as input a parameter $q$, and seek to provide for only $\lfloor qn \rfloor$ of the clients, thus ignoring the remaining ones. Formally, the new objective is to minimize $R$ such that $|\{p\in\mathcal{P}:d(S_p,F)\le R\}|\ge\lfloor qn\rfloor$.
    \item \textbf{Fairness}: many studies have shown that COVID-19 disproportionately affects some demographic groups \citep{tai2020disproportionate}. To counteract this, we seek to guarantee that different demographic groups have similar accessibilities to vaccines. As an example, when we solve the outliers formulation, we can guarantee that we are covering the same proportion of each demographic group when deciding the facility placements.
    \item \textbf{Capacity}: it is natural to assume that the number of vaccines that can be stored in each mobile facility is limited, say at most $L$. Therefore, in this setting, we need to guarantee that every chosen facility will have at most $L$ people assigned to it.
\end{itemize}

\section{Related Work}
\label{sec:related}

Due to its applications in a large number of domains, facility location and broader location theory is a very well-studied area; see, e.g., the surveys by 
\citep{drezner-book,turkoglu:faclocsurvey,Afshari2014ChallengesAS}. The general goal in this family of problems is to deploy facilities to provide the best possible service to a set of clients. A huge number of objectives have been considered, along with a plethora of variations such as fairness variants and online or stochastic versions.
The \prob{} problem we study here is a generalization of the well-known $k$-center problem, where the goal is to open at most $k$ centers while minimizing the maximum distance of a point to its closest center. For this simple clustering setting, there exist efficient $2$-approximation algorithms \citep{Hochbaum1985, Gonzalez1985}. Furthermore, it is shown that unless P=NP this is the best achievable approximation ratio \citep{Hochbaum1986}.

Location theory problems have also been considered in the area of healthcare, e.g.,~\citep{nedjati:ijmo12,meskarian:plos17,devries-pom20,Afshari2014ChallengesAS}. A lot of this work has been focused on placing mobile clinics or temporary facilities to ensure good service, especially in resource-poor countries. As mentioned in~\citep{Afshari2014ChallengesAS}, the healthcare domain poses new challenges for location theory, such as uncertainty, reliability, operation efficiency, patient safety, and cost-effectiveness. Prior work has generally not considered the mobility of clients at a detailed scale, which provides more flexibility in deploying facilities. Our formulation of \prob{} explicitly models human mobility, thus providing a realistic framework for public health agencies in their response efforts.

\section{Hardness Result}

For our hardness result, we use the following problem studied in \citet{anegg2020}, named $\gamma$-Colorful $k$-Center or $\gamma$C$k$C for short. This problem is a generalization of the outliers version of $k$-center: in addition to the classical constraints, colors (representing demographic groups) are assigned to each client and the problem requires that a sufficient number of points of each color is covered. The formal definition is given below:

\begin{definition}
Let $\gamma\in\mathbb{Z}_{\geq1}$ be the number of colors, $k \in \mathbb{Z}_{\geq 1}$ be the budget, $\mathcal{C}$ be a set of points in a metric space, and $d:\mathcal{C}\times\mathcal{C}\xrightarrow{}{\mathbb{R}_{\ge 0}}$ be the distance function on $\mathcal{C}$. For each $\ell \in [\gamma]$, let $\mathcal{C}_{\ell} \subseteq \mathcal{C}$ be the points with color $\ell$ and let $m_{\ell}\in\mathbb{Z}_{\ge 1}$ be the number of points with color $\ell$ which need to be covered. $\gamma$C$k$C asks for the minimum radius $R$ together with a set $F \subseteq \mathcal{C}$ with $|F|\le k$, such that
at least $m_{\ell}$ points of $\mathcal{C}_{\ell}$ are covered within distance $R$ by $F$. Formally, if $B(F,R) = \{j \in \mathcal{C}: d(j,F) \leq R\}$ then we want $|B(F,R) \cap \mathcal{C}_{\ell}| \geq m_{\ell}$ for every $\ell \in [\gamma]$.
\end{definition}

In \citet{anegg2020} the authors prove the following hardness result, which we use to prove a hardness result for our problem later on.

\begin{lemma}\label{thm:anegg}
When $\gamma$ is not a constant, there exist instances of $\gamma$C$k$C with $m_{\ell} = 1$ for all $\ell \in [\gamma]$, such that if $R^*$ is the optimal value of the instance, the following hold:
\begin{itemize}
    \item For any $\rho > 0$, it is NP-hard to find $F \subseteq \mathcal{C}$ with $|F| \leq k$ and $|B(F,\rho R^*) \cap \mathcal{C}_{\ell}| \geq m_{\ell}$ for all $\ell $. In words, it is NP-hard to devise any approximation algorithm for $\gamma$C$k$C.
    \item For any $\rho > 0$ and $\epsilon \in (0,1)$, it is NP-hard to find $F \subseteq \mathcal{C}$ with $|F| \leq (1-\epsilon) \ln \gamma \cdot k$ and $|B(F,\rho R^*) \cap \mathcal{C}_{\ell}| \geq m_{\ell}$ for every $\ell \in [\gamma]$. In words, it is NP-hard to devise any bicriteria approximation for $\gamma$C$k$C, whose chosen centers will be at most $(1-\epsilon) \ln \gamma \cdot k$.
    \item These problematic instances consist of points on a line.
\end{itemize}
\end{lemma}

\begin{remark}
Regarding the second statement in Lemma \ref{thm:anegg}, the authors of \citet{anegg2020} show a bicriteria hardness result in terms of $\log |\mathcal{C}|$ and not $\ln \gamma$. However, a closer look into their proof reveals that the claim mentioned above follows trivially. We choose to present this form of the claim because it better fits our narrative later on in the paper. 
\end{remark}

\begin{theorem}\label{thm:red}
There exists a bicriteria preserving reduction of the problematic instances of $\gamma$C$k$C described in Lemma \ref{thm:anegg} to instances of \prob{} with $|\mathcal{P}| = \gamma$. Specifically, any $(\rho, \alpha)$-bicriteria approximation for \prob{} translates to a $(\rho, \alpha)$-bicriteria approximation for the problematic instances of $\gamma$C$k$C.
\end{theorem}

\begin{proof}
Let $(\mathcal{C}, \mathcal{C}_1, \hdots, \mathcal{C}_\gamma, k, m_1, \hdots, m_\gamma)$ be a problematic instance of $\gamma$C$k$C as described in Lemma \ref{thm:anegg}, and recall that this instance has $m_{\ell} = 1$ for all $\ell \in [\gamma]$. We will now construct an instance of \prob{} as follows. The metric space for \prob{} will be the same as in the $\gamma$C$k$C problem. That is, we assume we have points $\mathcal{C}$ with a distance function $d$ on them. For every $\ell \in [\gamma]$ construct a client $p_{\ell}$, and set $S_{p_{\ell}} = \mathcal{C}_{\ell}$. The set of locations $\mathcal{S}$ for \prob{} where we can place facilities will be the set of locations $\mathcal{C}$ of $\gamma$C$k$C, and the value $k$ will stay the same for the two problems. 

Consider now the optimal solution $F^*$ of the $\gamma$C$k$C instance and its corresponding value $R^*$. We claim that $F^*$ is a feasible solution for the constructed \prob{} instance, and its value for that is exactly the same. This is easy to see because $|F^*| \leq k$, $|B(F^*,R^*) \cap \mathcal{C}_{\ell}| \geq 1$ for every $\ell \in [\gamma]$, and $\mathcal{C}_{\ell}$ are exactly the locations visited by client $p_\ell$. Hence if $R_{OPT}$ is the value of the optimal solution to the the constructed \prob{} instance, we have $R_{OPT} \leq R^*$.

Take now any $(\rho, \alpha)$-bicriteria solution $F$ for \prob{}. At first we trivially have $|F| \leq \alpha k$. Moreover, for every $\ell$ we can express $d(F,S_{p_\ell}) \leq \rho R_{OPT}$ (the condition guaranteed by the $(\rho, \alpha)$-bicriteria solution $F$ for \prob{}) as $|B(F,\rho R_{OPT}) \cap \mathcal{C}_{\ell}| \geq 1$. Finally, because $R_{OPT} \leq R^*$, we have $B(F,\rho R_{OPT}) \subseteq B(F,\rho R^*)$. Hence, $|B(F,\rho R^*) \cap \mathcal{C}_{\ell}| \geq |B(F,\rho R_{OPT}) \cap \mathcal{C}_{\ell}| \geq 1$ for every $\ell \in [\gamma]$. The latter concludes the bicriteria preserving reduction.
\end{proof}

\begin{corollary}\label{cor-2}
Even when the metric space is the Euclidean line, we have the following for \prob{} (unless P$=$NP):
\begin{enumerate}
    \item No approximation algorithm exists.
    \item Any bicriteria algorithm must use at least $k \ln n$ facilities.
\end{enumerate}
\end{corollary}
\section{Algorithms}
\label{sec:bicriteria}

In this section, we introduce efficient methods which give (approximately) optimal facility placements, despite the hardness results. We also show how to extend each of our algorithms to ignore outliers, incorporate fairness constraints, and restrict the capacity of each facility.

\subsection{Fixed-Parameter Tractability}
Let $U = \bigcup_{p \in \mathcal{P}} S_p$ denote the set of all the locations visited by the set of clients and $u = |U|$ be the number of locations in this set. Due to potential privacy concerns, we can assume that the client locations we have access to only include large public areas in the county such as malls, shopping centers, etc. Hence, it is reasonable to conclude that $u$ is a fixed parameter, which we assume ranges from $15-30$. Given this fixed parameter, we develop an efficient algorithm for our problem.

The main observation here is the following: consider an instance of \prob{} and let $F^*$ be its optimal solution, whose maximum radius we denote by $R^*$. For each $p \in \mathcal{P}$, we know that $d(F^*, S_p) \leq R^*$, and hence there must exist a location $i_p \in S_p$ with $d(i_p, F^*) \leq R^*$. See now that $\{i_p ~|~ p \in \mathcal{P}\} \subseteq U$, and therefore $|\{i_p ~|~ p \in \mathcal{P}\}| \leq u$. The latter implies that we can guess, via an exhaustive search, the set $\{i_p ~|~ p \in \mathcal{P}\}$ in time at most $2^u$ (recall that since $u$ is considered a fixed parameter, $2^u$ is thought of as a small constant). Let $A$ be the correct guess for that set; we can think of $A$ as the set of locations through which the optimal solution covers every client within distance $R^*$. Given $A$, we see that the problem of computing $F^*$ reduces in a straightforward manner to the well-known $k$-supplier problem \citep{Hochbaum1985}.

In $k$-supplier we have a set of points $\mathcal{X}$ and a set of locations $\mathcal{Y}$ in a metric space with distance function $d$. The goal is to choose  $C \subseteq \mathcal{Y}$ with $|C| \leq k$, such that the maximum distance of any point in $\mathcal X$ to its closest location of $S$ is minimized. Hence, after correctly guessing $A$, we create an instance of $k$-supplier where the points are the ones in $A$, and the set of locations $\mathcal{Y}$ is $\mathcal{S}$. The previous discussion shows that $F^*$ is a solution for this $k$-supplier instance, and its maximum radius will again be $R^*$. Moreover, any $\rho$-approximate solution to the $k$-supplier instance will trivially be a $\rho$-approximate solution for \prob{}. Using the $3$-approximation algorithm from \citet{Hochbaum1985} proves the following theorem.

\begin{theorem}
Algorithm \ref{alg:fpt} yields a $3$-approximation algorithm for \prob{} and runs in time $2^u \poly(n, |\mathcal{C}|)$.
\end{theorem}
\begin{algorithm}
\caption{\fpt{}}
\label{alg:fpt}
\begin{algorithmic}[1]
\FOR{$A\in 2^U:|A\cap S_p|\neq0,
\forall p\in \mathcal{P}$}
\STATE Obtain locations $F_A$ by running the $k$-supplier algorithm on the appropriate instance discussed above.\\
\STATE Calculate the objective value for $F_A$.
\ENDFOR
\STATE Pick the $F_A$ with the smallest objective value.
\end{algorithmic}
\end{algorithm}
Moving forward, we see that the same approach of guessing the correct set of client locations $A$ will also apply in different settings. In fact, the only thing that may differ is the need for an alternative $k$-supplier algorithm that can incorporate the specific constraints of each unique setting; we survey some of these settings below.

\textbf{Outliers:} 
To modify our algorithm so that it only considers some fraction $q$ of the population, we only need to change the objective value evaluated in line 3 of Algorithm \ref{alg:fpt}. To improve efficiency, we can also only consider guesses $A$ that contain locations from at least $\floor*{q n}$ clients since the correct guess $A$ contains locations from at least $\floor*{q n}$ clients. If we then feed $A$ to the $k$-supplier algorithm in the exact same manner before, we will get a $3$-approximation. 

\begin{corollary}
After changing the objective evaluated in line 3 to the partial objective, Algorithm \ref{alg:fpt} gives a $3$-approximation for \prob{} with outliers.
\end{corollary}

\textbf{Fairness:} Although our algorithm provides an upper bound guarantee for the maximum distance to a facility, the facility placement may significantly differ between individuals, with some having a facility right next to them, while others need to travel the whole $3R^*$ guarantee. Luckily, the vaccine centers can vary from week to week or even day to day. Thus, we can use a randomized algorithm such as the one given in \citet{harris2020approximation}, to guarantee that the re-provisioning of facilities over the course of many tries will provide an improved per-point guarantee on expectation. Hence, we treat the clients stochastically fairly.

\begin{corollary}
When using the algorithm from \citet{harris2020approximation} instead of a simple $k$-supplier algorithm, Algorithm \ref{alg:fpt} is able to output a distribution $\tilde\Omega$ such that $\forall p\in \mathcal{P}$, we have $\mathbb{E}_{F\sim\tilde\Omega}[d(S_p,F)]\le (1+2/e) R^*$ and $\Pr[d(S_p,F)\le 3 R^*]=1$.
\end{corollary}

\textbf{Capacity:} 
In this case, we assume that each facility we use has a capacity $L$, i.e., at most $L$ clients can be assigned to it in any solution. Once again, the FPT process we described earlier suffices to solve the problem. Specifically, we can think of the set $A$ as the locations through which the clients of $\mathcal{P}$ receive their service. Hence, as we did for the regular case, we can create an instance of $k$-supplier where the points requiring service are those of $A$, but this time each location of $\mathcal{Y}$ for the $k$-supplier instance will have a capacity $L$. In other words, this will be an instance of capacitated $k$-supplier. Furthermore, the optimal solution of capacitated \prob{} will be a solution of the same value for the capacitated $k$-supplier instance. Finally, it is also trivial to see that any $\rho$-approximate solution for capacitated $k$-supplier instance, will yield a $\rho$-approximate solution to capacitated \prob{}. 

\begin{corollary}
When using the algorithm from \citet{an2013centrality} instead of a simple $k$-supplier algorithm, Algorithm \ref{alg:fpt} is an 11-approximation for capacitated \prob{}.
\end{corollary}

\subsection{Covering Algorithm} 

In Corollary \ref{cor-2}, we show that any bicriteria algorithm needs to open at least $k\ln n$ facilities in order to give a bounded approximation guarantee. Here, we show that this is essentially tight: we give an algorithm that outputs a set of locations of size at most $k(\ln n + 1)$, while guaranteeing that the objective value is at most that of an optimal solution. 

Consider the related problem, which we call \cover{}, in which instead of optimizing the radius $R$ given a budget $k$, we are given a target radius $R$ and want to choose a set $F\subseteq \mathcal{S}$ which minimizes $|F|$ and guarantees that $d(S_p,F)\le R$ for each $p\in\mathcal{P}$. Notice that this is just a standard Set Cover problem, where the sets are $\{p\in \mathcal{P}: d(S_p, j)\le R\}$ for each $j\in \mathcal{S}$ and the universe consists of the clients $\mathcal{P}$. Using a known greedy algorithm for Set Cover \citep{approx}, we have an $H_n$-approximation algorithm for \cover{}, where $H_n\le \ln{n}+1$ is the $n$-th harmonic number. 

For generality, we will show how any $\alpha$-approximation algorithm for Set Cover yields an $(1,\alpha)$-bicriteria algorithm for \prob{} via a reduction to \cover{}. First, note that the optimal radius $R^*$ for an instance of \prob{} is always the distance between some $j \in \mathcal{C}$ and some $i \in \mathcal{S}$. Hence, there are at most polynomially many options for it, specifically $|\mathcal{C}|\cdot |\mathcal{S}|$. For each such option $R$, we create the corresponding instance of \cover{} and run the set cover algorithm on it. The final guarantees follow from the iteration when $R=R^*$. Observe at this point that we can speed up the whole process by performing a binary search in order to find $R^*$, and thus avoid the previously described exhaustive search.

\begin{algorithm}[H]
\caption{\cover{} Search}
\label{alg:algorithm}
\begin{algorithmic}[1]
\STATE Binary search on the sorted list $\{d(i,j) : j \in \mathcal{C}, i \in \mathcal{S}\}$, and let the current guess be $R$:
\STATE\quad Use $R$ to create the proper instance of \cover{}. 
\STATE\quad Obtain $\alpha$-approximate solution $F_R$ for that instance.
\STATE\quad If $|F_R| > \alpha \cdot k$, increase $R$; else, decrease $R$.
\STATE Output $F_R$ for the minimum $R$ such that $F_R \leq \alpha \cdot k$.
\end{algorithmic}
\end{algorithm}

\begin{theorem}\label{thm:bicriteria}
Given an $\alpha$-approximation algorithm for set cover, Algorithm \ref{alg:algorithm} gives an $(1, \alpha)$-bicriteria algorithm for \prob{}.
\end{theorem}
\begin{proof}

Let $R^*$ be the objective value for the optimal solution $F^* \subseteq \mathcal{S}$ of \prob{}, where $|F^*| \leq k$. We wish to show that \cover{} Search finds a radius $R$ in the list such that $|F_R|\le \alpha\cdot k$ and $R\le R^*$. Consider an iteration of the binary search where the radius guess is $R$. Suppose $R \geq R^*$; then there must exist a solution of \cover{} of size at most $k$. The $\alpha$-approximation algorithm will therefore output a set $F_R$ with $F_R\le\alpha\cdot k$ and $R$ will decrease. If $R < R^*$, then we either find a solution with $F_R\le\alpha\cdot k$, or we increase $R$ and move closer to $R^*$. Finally, since $R^*$ is in the list $\{d(i,j) : j \in \mathcal{C}, i \in \mathcal{S}\}$, the binary search necessarily finds some $R\le R^*$ with $F_R\le\alpha\cdot k$.
\end{proof}

As in the case of our \fpt{} algorithm, we can easily extend Algorithm \ref{alg:algorithm} in order to accommodate different settings. The only difference here lies at step 3, where instead of a classic Set Cover algorithm we can run a different algorithm.

\textbf{Outliers:} In order to modify our algorithm to only consider some fraction $q\in(0,1)$ of the population, we can use some $\alpha$-approximation algorithm for the Partial Set Cover problem, where the goal is to cover at least a $q$-fraction of the universe elements. Hence, we naturally consider a variant of \cover{}, which we call Partial \cover{}, that requires only $\floor*{qn}$ points to be covered by balls of radius $R^*$. Trivially, Partial \cover{} is a special case of Partial Set Cover. Then the approach we described previously remains the same: we can guess the optimal radius $R^*$ and obtain an $\alpha$-approximate solution $F_{R^*}$ for the corresponding Partial \cover{} instance. This solution will be optimal for \prob{} with outliers, while placing at most $\alpha k$ facilities. In particular, we have the following corollary of Theorem \ref{thm:bicriteria}.

\begin{corollary}
When using the greedy algorithm for Partial Set Cover \citep{approx}, Algorithm \ref{alg:algorithm} gives a $(1,H_{\lfloor qn\rfloor})$-bicriteria algorithm for \prob{} with outliers.
\end{corollary}

\textbf{Fairness:} 
When solving \prob{} with outliers, the algorithm may view some demographic groups as outliers more often than others. To mitigate such possibilities, we can use an algorithm for the Partition Set Cover problem \citep{inamdar2018partition} to guarantee that a large proportion of each demographic group gets coverage. For example, we can guarantee that the algorithm considers a proportional number of people from each (demographic) group when choosing the vaccine center locations. The following approximation guarantee will then follow directly from \citet{inamdar2018partition} and the outliers reduction before. We remark that even though the word ``partition" is in the name of the problem, the results of \citet{inamdar2018partition} extend to the case of overlapping demographic classes.

\begin{corollary}
Let $C_t\subseteq\mathcal{P}$ for $t\in[r]$ be (not necessarily disjoint) demographic classes and let $0\le p_t\le |C_t|$ be the coverage requirements for each class. Using the algorithm of \citet{inamdar2018partition} at step 3, Algorithm 2 gives a $(1,O(\log n)+\log{r})$-bicriteria algorithm while satisfying the coverage constraints.
\end{corollary}

\textbf{Capacity:} As before, we assume that each facility we use has capacity $L$. We see that our general framework is still applicable: we can modify our algorithm to satisfy these capacity constraints by replacing the Set Cover algorithm with a Capacitated Set Cover algorithm when solving the \cover{} problem. In fact, \citet{wolsey1981analysis} shows that a greedy algorithm (similar to the one for Set Cover) still gives a $\log{n}+1$ approximation algorithm in this more general case.

\begin{corollary}
When using the greedy algorithm given in \citet{wolsey1981analysis} for Capacitated Set Cover, Algorithm 2 gives a $(1,\log{n}+1)$-bicriteria algorithm while satisfying the capacity constraints.
\end{corollary}

\textbf{Budget:} In the previous algorithms which solve \cover{} as a subroutine, we violate the budget constraint $k$ by a non-trivial multiplicative factor, which is a practical consideration that needs to be addressed. Luckily, it has been shown that the greedy algorithm and other heuristics for Set Cover have very small approximation ratios in practice \citep{lan2007an}. In fact, many real-life instances of Set Cover are solved optimally or near optimally by the greedy algorithm \citep{GROSSMAN199781}. Given this empirical result (which we also validate for our instances of the Set Cover problem), we get $\alpha=1$ in our experiments when running \cover{} Search. In particular, if we solve the \cover{} problem using a commercial mixed-integer linear program (MILP) solver \citep{gurobi,ortools}, we can solve the original problem to optimality. 
We emphasize that this is a non-trivial contribution: directly formulating \prob{} as an MILP requires $\Theta(n^3)$ constraints, and we cannot even initialize the solver using or-tools. In contrast, the Set Cover MILP only has $\Theta(n)$ constraints, which can be solved efficiently using commercial solvers. Hence, when using an MILP solver to solve \cover{}, Algorithm 2 yields a practical solution for solving \prob{} optimally.

\begin{table*}[t]
    \centering
    \caption{Network Information}
    \begin{tabular}{c|cccccc}
    \toprule
     &  & Activity & Residential & Maximum & Measured \\
     & Clients  & Locations & Locations & Activity & Diameter (km) \\
    \midrule
    Charlottesville City & 33156 & 5660 & 10038 & 9952 &  8.12\\
    Albemarle County & 74253 & 9619 & 32981 & 24506& 61.62\\
    \bottomrule
    \end{tabular}
    \label{tab:my_label}
\end{table*}
\section{Experiments}

\subsection{Experimental Setup}

\textbf{Data}: We run our experiments on the mobility data from Charlottesville City and Albemarle County in Virginia. For these counties, we use synthetic data constructed from the 2019 U.S. population pipeline (see \citet{chen2021, machi2021} for details). This dataset was constructed by tracking the week-long activity of county residents. Each resident is represented by a sequence of activities, where each activity is described by duration, type, and location in the county. The locations are given in geodetic coordinates and are categorized as either a residential or activity (non-residential) location. From this dataset, we can extract the locations visited by individuals residing in the county and set all activity locations as potential facility placements. A summary of the dataset is given in Table \ref{tab:my_label}.

\textbf{Baselines}: We compare our algorithms with two heuristics: \homecenter{} and \pop{}. In \pop{}, we open vaccination centers at the $k$ most visited locations. We set \pop{} as the baseline because it is related to the current heuristic used by the Virginia Department of Health. In \homecenter{}, we run k-supplier to place facilities at locations that minimize the maximum distance from client homes. We compare with this baseline in order to show the importance of considering mobility when placing the vaccination centers.

\textbf{Objective}: Recall that our objective is to minimize the maximum distance any client needs to deviate from their path to reach some facility. Since our location data is given in the geographic coordinate system, we approximate the Earth as a sphere and use geodesic distance as our metric. In Section 6.2, we notice that there is a sharp drop in the objective value if we only consider $99\%$ of the population. As a result, we also evaluate the objective value of our algorithms when $5\%$ of the people are considered outliers.

\textbf{FPT details:} When using \fpt{} in our experiments, we pick $u=15$ locations that cover the largest portion of the population (as given by the greedy algorithm for the Maximum Coverage problem). We then run \fpt{} using only knowledge of these $u$ locations. The locations chosen are all popular public activity locations, so we have a limited amount of privacy violation. As a result, the performance of \fpt{} is weaker on the full objective, but remains strong on partial coverage (the outliers formulation). It is important to note that even though we limit the knowledge of client-visited locations, \fpt{} can still choose to place facilities at any activity location in the dataset. For more details on our implementation of FPT and the  experiments, see our GitHub\footnote{\href{https://github.com/Ann924/MobileFacility}{https://github.com/Ann924/MobileFacility}}.

\subsection{Client Coverage Performance Analysis}

First, we directly compare the performances between our algorithms and the baselines. Because our objective value is defined by the maximum distance any client must travel to reach their closest facility, it does not yield insight into the distribution of travel distances. For this reason, we also assess how large the radius around our placed facilities must be to cover proportion $p$ of the clients, for $p\in[0.8,1.0]$. 

As seen in Figure \ref{fig:percentile_cville}, facility placements from \homecenter{} and \cover{} result in better full objectives while facility placements from \pop{} and \fpt{} result in better partial objectives. This has a simple explanation: the former two algorithms are forced to consider outliers since they directly optimize the objective while the latter two optimize over only a portion of the population. As a result, it makes sense to compare \homecenter{} with \cover{} and \fpt{} with \pop{}. We note that if we instead used the outliers version of \homecenter{} and \cover{}, we may not see this disparity.

\begin{figure}[h]
    \centering
    \resizebox{0.9\columnwidth}{!}{\includegraphics{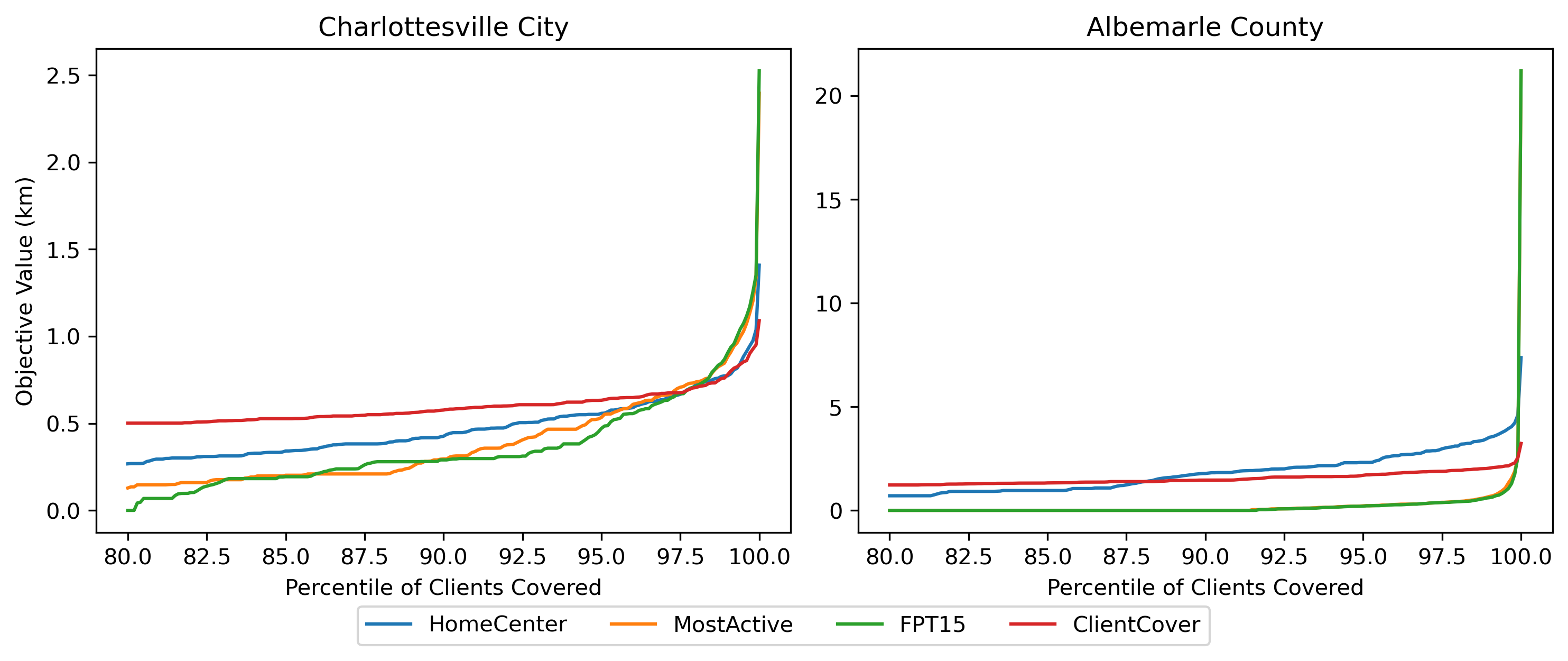}}
    \caption{Charlottesville k = 10, Albemarle k = 20}
    \label{fig:percentile_cville}
\end{figure}
\subsection{Tradeoff between Radius and Budget}

In addition to evaluating the performance of our algorithms at the current budget, it is important to evaluate the sensitivity of our algorithms to an increase in budget. That is, we want to know how much the objective value would decrease if the county allocated more resources to deploy a greater number of mobile facilities. This knowledge can influence policy decisions: when an increase budget yields a sharp decrease in objective (rather than a small decrease), the government has more incentive to fund another vaccination center.

As seen in Figure \ref{fig:tradeoff}, there is generally a sharp decrease in the objective value when the budget is less than 6 for Charlottesville and 9 for Albemarle. As the budget increases past those thresholds, the marginal returns become so diminished that increasing the budget hardly changes the objective value. This is especially prominent in the full objective performance of \fpt{} and \pop{}. Hence, it is natural to recommend budgets of 6 and 9 to the Charlottesville and Albemarle government, respectively.

Additionally, we wish to bring attention to the overall poor performance of \homecenter{} in these experiments. It is consistently outperformed by \cover{} for the full objective and is outperformed by every algorithm when evaluating the partial objective. Furthermore, \homecenter{} does not exhibit strong budget sensitivity when assessing the 95\% objective. On Albemarle, its performance plateaus around an objective value of 1.5km, which is more than 3 times larger than the objective of the algorithms that consider mobility patterns.

A seemingly weird result from the experiment is the tradeoff curve for \homecenter{}. Though there is a general downward trend in the objective value as the budget increases, there are cases in each county where increasing the budget results in an increase in the objective value. This contradictory phenomenon is caused by the limited correlation between the distance to homes and our objective; as a result, noise/luck has a considerable effect. The noisiness of \homecenter{} emphasizes the importance of our work of modeling mobile populations. 

\begin{figure}[h]
    \centering
    \resizebox{0.9\columnwidth}{!}{\includegraphics[]{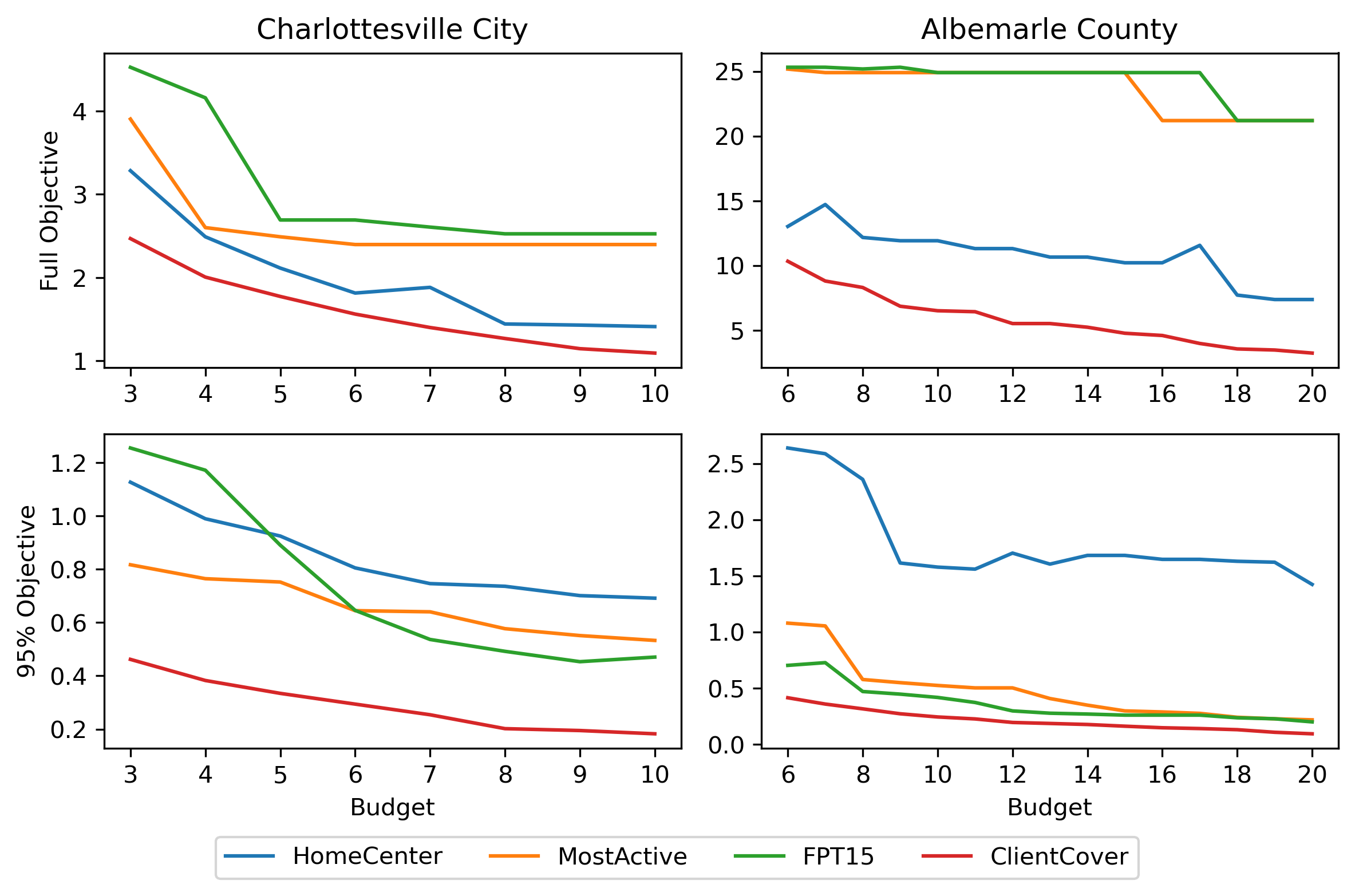}}
    \caption{Tradeoff between maximum distance needed to travel and the number of vaccine centers placed. In the 95\% objective plots, \cover{} and \homecenter{} are run with the outliers formulation.}
    \label{fig:tradeoff}
\end{figure}

\subsection{A Kernel Property}

Through our experiments, we notice a nice (empirical) property of the vaccine center locations selected by some of our algorithms. Imagine a case where we (the government) have the funds to place five mobile vaccine centers and we use our algorithms to pick the five locations to place them. Then, after two weeks, the government decides that the disease is causing too much economic devastation and, in turn, funds three more mobile vaccine centers. When we ask our algorithms to place the eight vaccine centers (approximately) optimally, it turns out that the eight chosen locations will often contain the original five chosen locations as a subset. The original five locations are then called a kernel.

\begin{table}[H]
\centering
\parbox{.45\linewidth}{
    \centering
    \caption{Charlottesville Kernel}
    \begin{tabular}{c| c c c c c}
    \toprule
          & \textsc{Home} &  & \textsc{Client}\\
          & \textsc{Center} & \fpt{} & \textsc{Cover}\\
          \midrule
         3 $\rightarrow$ 4  & 1& 1& 2\\
         4 $\rightarrow$ 5  & 1& 1& 4\\
         5 $\rightarrow$ 6  & 2& 1& 5\\
         6 $\rightarrow$ 7  & 3& 2& 6\\
         7 $\rightarrow$ 8  & 0& 1& 6\\
         8 $\rightarrow$ 9  & 0& 1& 5\\
         9 $\rightarrow$ 10  & 3& 1& 7\\
    \bottomrule
    \end{tabular}
    \label{tab: cville_kernel}}
    \hfill
    \parbox{.45\linewidth}{
    \centering
    \caption{Albemarle Kernel}
    \begin{tabular}{c| c c c c}
    \toprule
          & \textsc{Home} &  & \textsc{Client}\\
         & \textsc{Center} & \fpt{} & \textsc{Cover}\\
         \midrule
        6 $\rightarrow$ 7&3&1&5\\
        7 $\rightarrow$ 8&1&2&6\\
        8 $\rightarrow$ 9&1&2&7\\
        9 $\rightarrow$ 10&1&1&8\\
        10 $\rightarrow$ 11&4&3&5\\
        11 $\rightarrow$ 12&1&1&9\\
        12 $\rightarrow$ 13&1&2&4\\
        13 $\rightarrow$ 14&0&1&7\\
        14 $\rightarrow$ 15&5&2&7\\
        15 $\rightarrow$ 16&1&1&8\\
        16 $\rightarrow$ 17&1&0&12\\
        17 $\rightarrow$ 18&7&0&8\\
        18 $\rightarrow$ 19&2&1&7\\
        19 $\rightarrow$ 20&0&0&8\\
    \bottomrule
    \end{tabular}
    \label{tab: albe_kernel}}
\end{table}

The ideal kernel property occurs when every set of chosen facilities of size $k$ contains the set of chosen facilities for budget $k-1$. In order to determine the presence of a kernel for each of our algorithms, we calculate the number of facilities chosen with budget $k-1$ that are not also chosen with budget $k$. These values populate Tables \ref{tab: cville_kernel} and \ref{tab: albe_kernel}, where the leftmost column denotes the budgets compared. By definition, \pop{} has the kernel property since it is a greedy algorithm. Our \fpt{} algorithm also (approximately) satisfies the kernel property while maintaining a stronger performance than the baseline. The remaining two algorithms do not exhibit the property: both \homecenter{} and \cover{} pick (almost) completely different locations upon increasing the budget. Because they require less relocation, \pop{} and \fpt{} have advantageous properties when the budget is adaptive and vaccine distribution is time-consuming.

We recognize that this is not necessarily applicable to our experimental setting, COVID-19, since transportation of vaccines is (relatively) easy in Virginia. However, for the Ebola outbreak in 2014, the kernel property was recognized as an important property to have since vaccine distribution was a much more costly process. Furthermore, we note that our algorithms are not explicitly designed to have this property; it is only empirically verified.

\subsection{Information Constraints}\label{sec:infconst}
In our previous experiments with \cover{}, we assumed that we had full knowledge of the locations each person visited throughout a day. Next, we wish to understand how fine-grained this data needs to be in order for \cover{} to outperform our other algorithms; this also addresses privacy concerns raised when using the exact mobility data of individuals. In order to model loss of detailed movement patterns, we cluster the locations within a given radius $r$ together and apply \cover{} on the resulting cluster centers. The details of the clustering algorithm can be found in our code, but the general idea is to define each location to be a potential cluster center and then use the greedy set cover algorithm to pick a minimum set of clusters centers that cover all original locations with radius $r$. Using this general method for both Charlottesville and Albemarle, we vary the radius $r$ between 100 and 600 meters to see how much privacy \cover{} can give while still maintaining a superior performance over the baselines.

\begin{figure}[h]
    \centering
    \includegraphics[width =0.7\columnwidth]{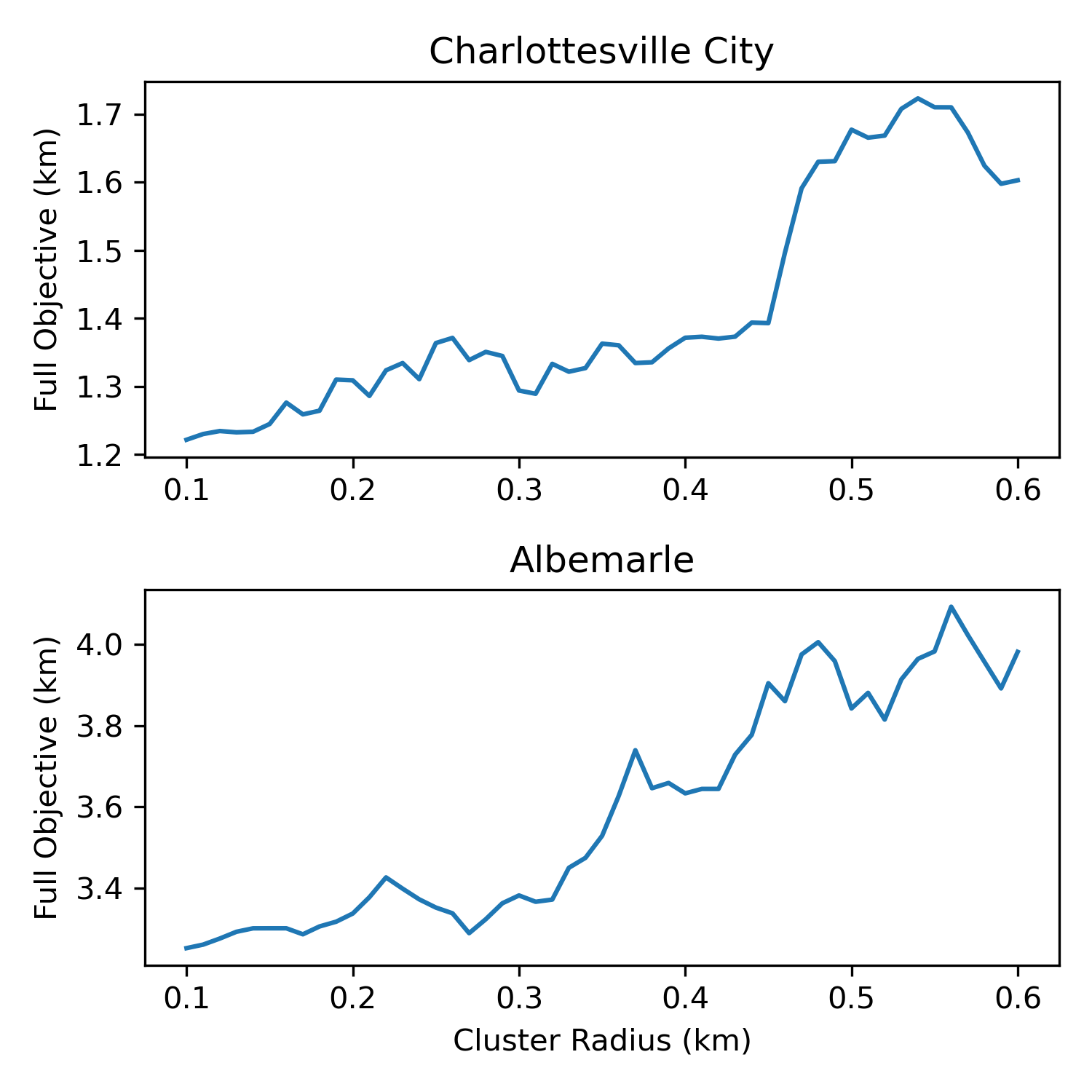}
    \caption{\cover{} performance under loss of information (k=20)}
    \label{fig: cluster_exp_cville}
\end{figure}

As we see in Figure \ref{fig: cluster_exp_cville}, clustering with a radius of 0.1--0.48 km on Charlottesville leads to a gradual increase in the objective value. At a clustering radius of 0.48 km, there is a sharp increase in the objective value from the facility placements. Even so, the resulting objective value is significantly smaller than 2.39 km, as obtained by \pop{}, and 2.52 km, as obtained for \fpt{}. Furthermore, even by clustering the data with a radius that is 45\% of the original \cover{} objective, we can still perform better than the \homecenter{} baseline. A similar trend occurs for Albemarle where the change in the solution value is relatively small when clustering from 0.10--0.35 km but grows rapidly after 0.35 km. We conclude that even when giving some privacy to individuals, \cover{} still performs much better than \fpt{}, \pop{}, and \homecenter{}.

\section{Conclusion}

Here, we introduce a generalization of the classical $k$-supplier problem where we consider the mobility of populations when placing facilities. We show that designing an approximation algorithm for this variant is NP-Hard, so we turn to fixed-parameter tractability and bicriteria approximation algorithms to get around our hardness result. Finally, we experimentally show the efficacy of our algorithms in comparison to natural baselines. Since we have demonstrated the importance of modeling mobile populations, a natural next step is to extend other variants of the facility location problem to this setting as well.

\textbf{Acknowledgements:}
We express our sincere thanks to the AAMAS referees for suggesting the experiments in Section~\ref{sec:infconst} and the extension of capacity constraints. We also thank members of the Biocomplexity COVID-19 Response Team and the Network Systems Science and Advanced Computing (NSSAC) Division for their thoughtful comments and suggestions related to epidemic modeling and response support. George Li, Aravind Srinivasan, and Leonidas Tsepenekas were supported in part by NSF award number CCF-1918749. Ann Li, Madhav Marathe, and Anil Vullikanti were supported by DTRA (Contract HDTRA1-19-D-0007), University of Virginia Strategic Investment Fund award number SIF160, National Institutes of Health (NIH) Grants 1R01GM109718, 2R01GM109718, OAC-1916805 (CINES), CCF-1918656 (Expeditions), CNS-2028004 (RAPID), OAC-2027541 (RAPID), IIS-1908530, IIS-1955797, and IIS-2027848.  The U.S. Government is authorized to reproduce and distribute reprints for Governmental purposes notwithstanding any copyright annotation thereon.

\bibliographystyle{named}
\bibliography{refs.bib}

\end{document}